\documentclass{article}

\usepackage{graphicx} 
\usepackage{subfigure} 
\usepackage{amsmath}
\usepackage{amssymb}
\usepackage{amsthm}

\usepackage{natbib}

\usepackage{algorithm}
\usepackage{algorithmic}

 \usepackage[accepted]{icml2011}

\newtheorem{theorem}{Theorem}
\newtheorem{lemma}{Lemma}

\newtheorem{definition}{Definition}

\usepackage{graphicx}
\usepackage{subfigure}

\def\figstochasticbudget{
\begin{figure*}[t]
\center
\includegraphics[width=\linewidth]{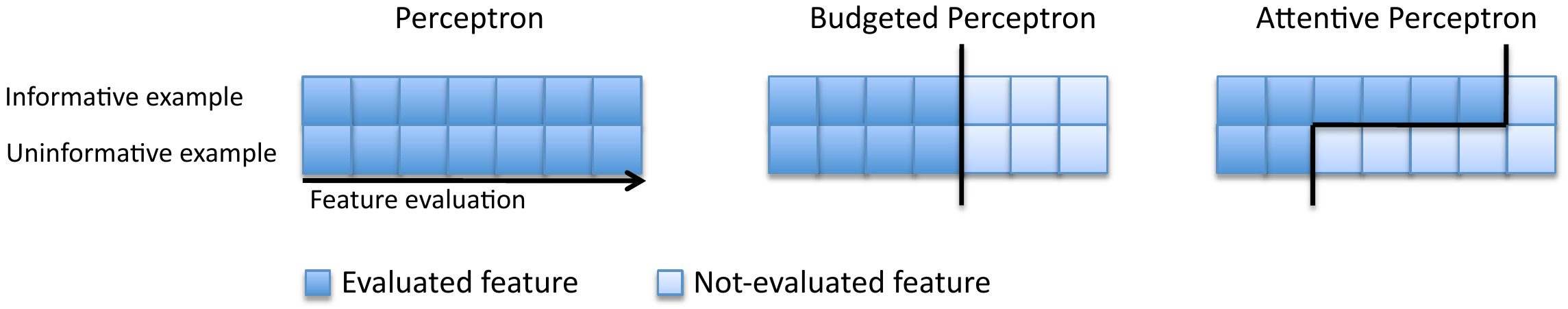}
\caption{Two examples are classified. The first is hard to classify, the second easy. The budgeted learning approach would evaluate the same number of features for both examples, whereas the stochastic would evaluate features according to how hard is the example to classify, while maintaining an average budget.}
\label{fig:stochastic-budget}
\end{figure*}
}

\def\figbbdecisionerror{
\begin{figure*}[t]
\centering
\mbox{
\subfigure[A simulation of the Brownian bridge boundary with $X_i \sim N(0.05,1)$. The boundary is conservative.]{\includegraphics[width=3in,height=2.7in]{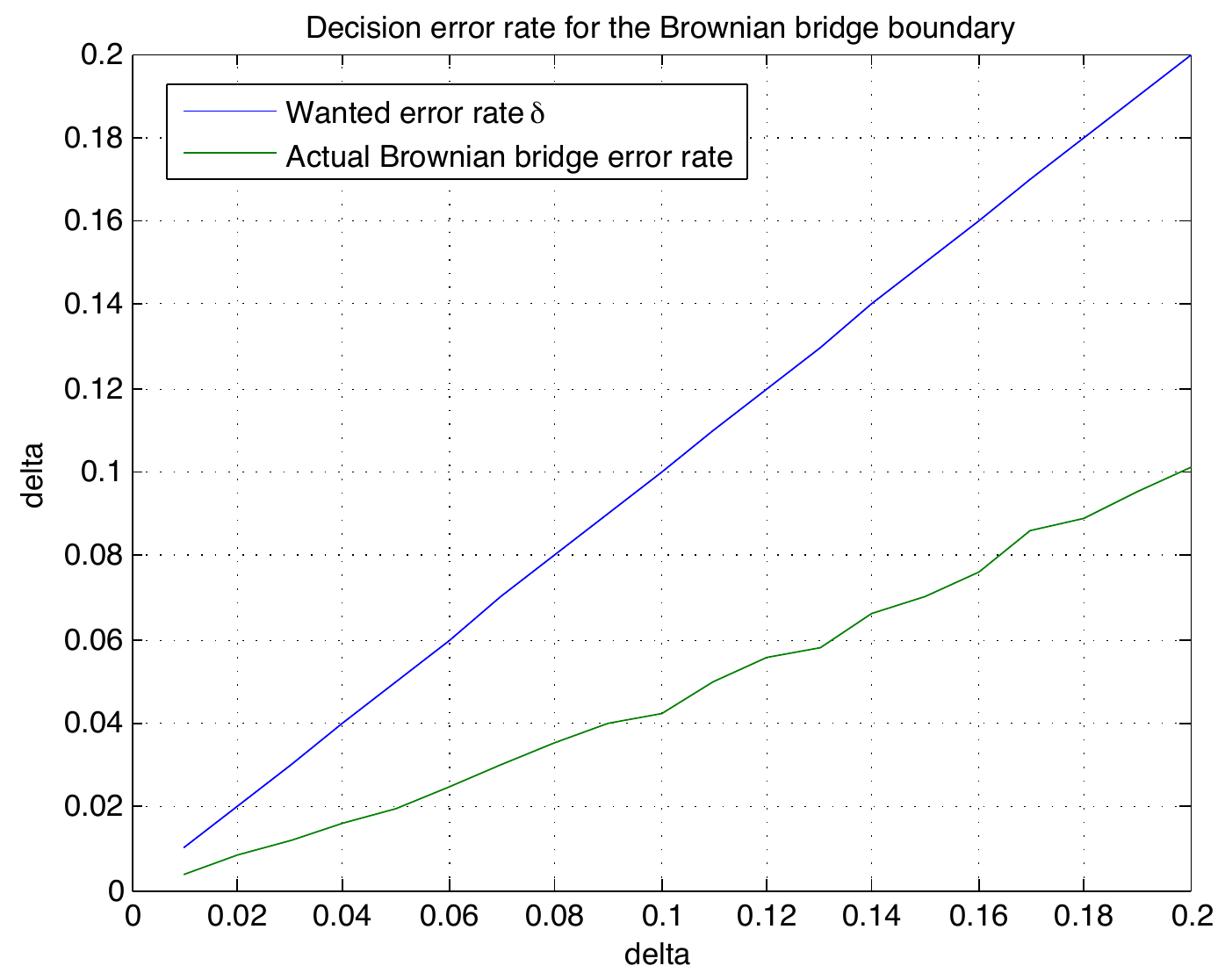} \label{fig:bb_error_rate}
} 
\quad
\subfigure[The boundary behaves similarly to what's expected from theory. It computes in the order of $O(\sqrt  n)$ features.]{\includegraphics[width=3in,height=2.7in]{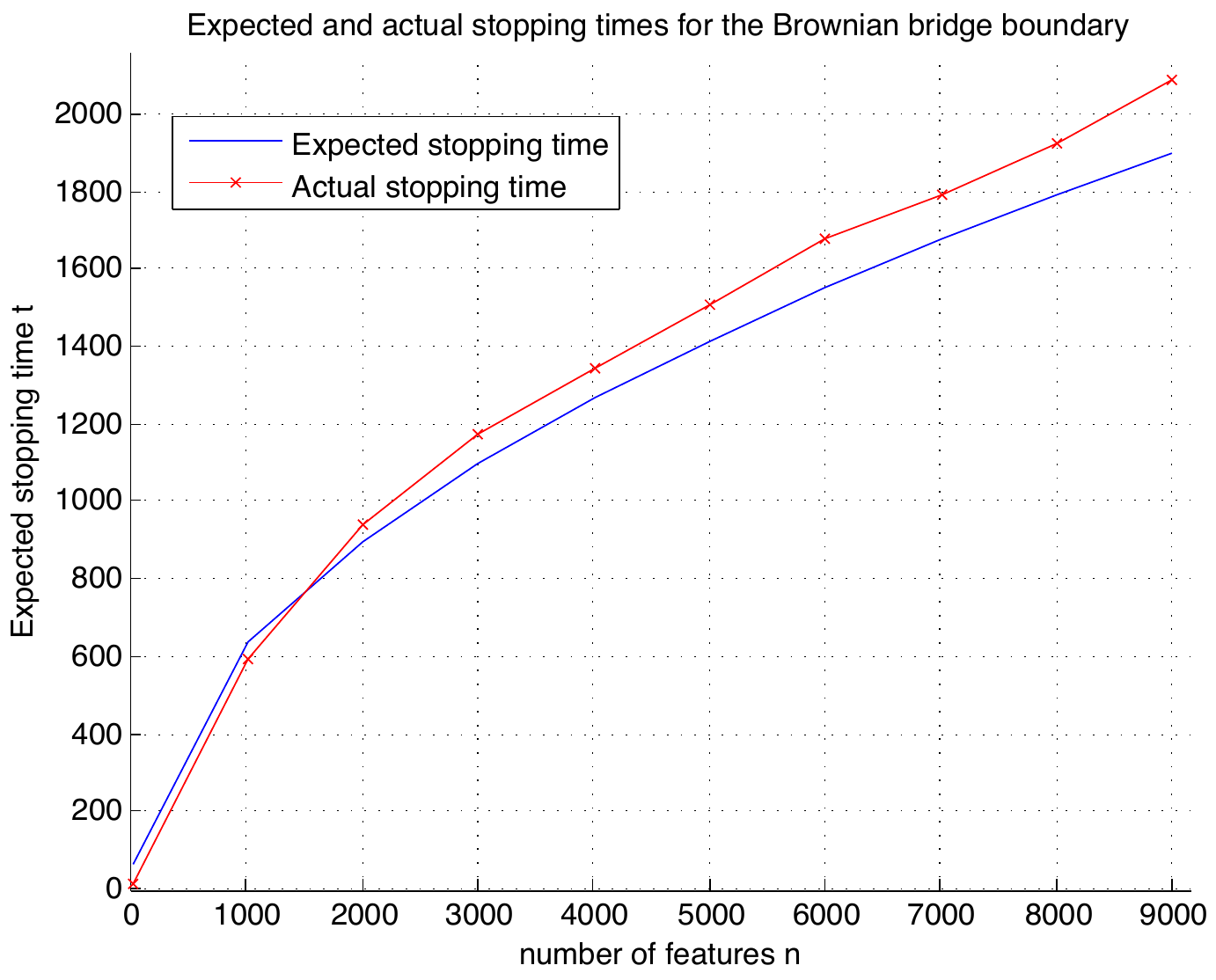} \label{fig:bb_stopping_time}
}
}
\caption{Performance of the Brownian bridge boundary.}
\end{figure*}
}

\def \figurecurtailedpegasosA{
\begin{figure*}[t]
\begin{center}
\includegraphics[height=2.5in,width=5in]{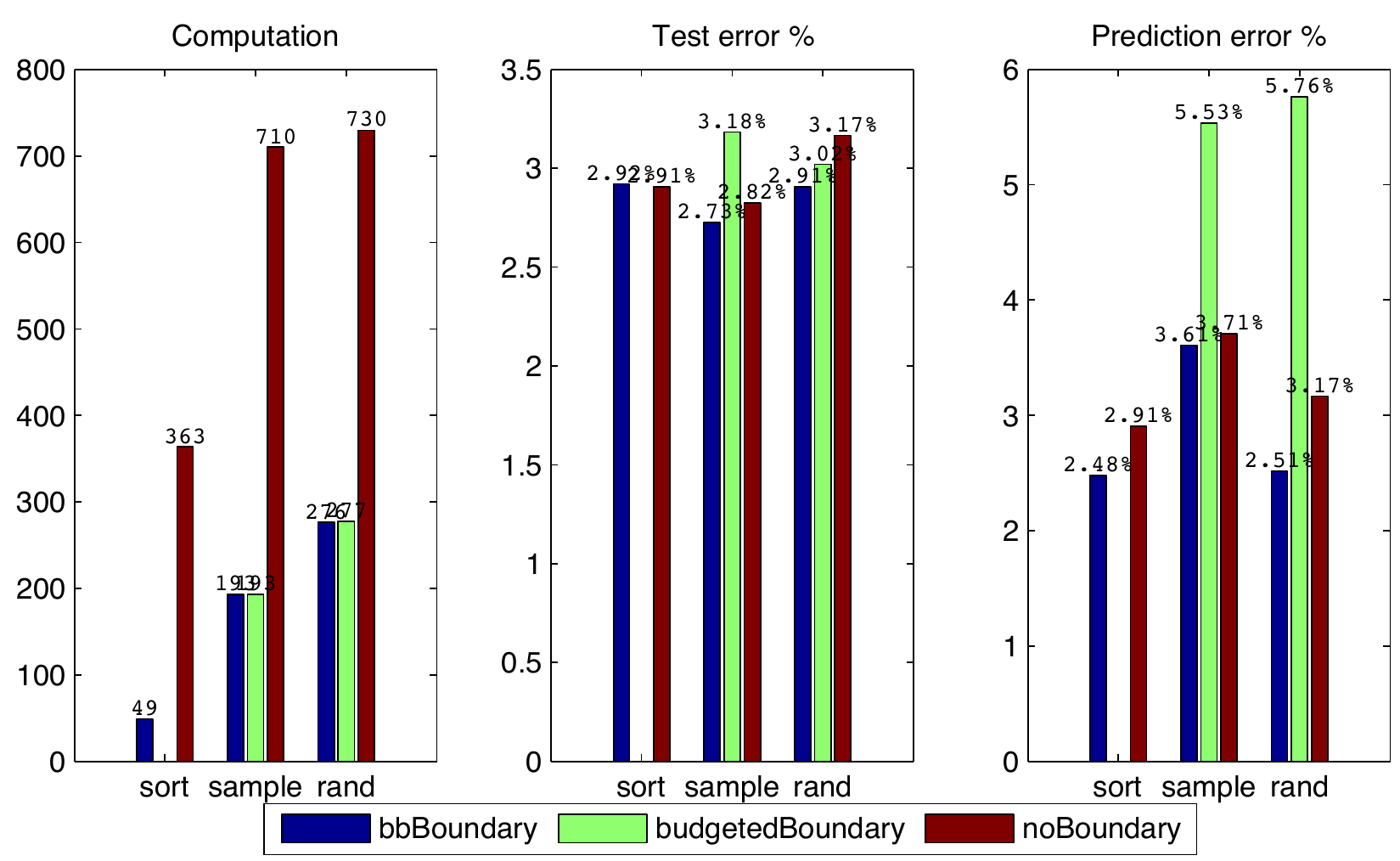}
\caption{Results for Attentive Pegasos, MNIST 2 vs 3, $\delta=10\%$. Our Brownian bridge decision boundary (blue) processes only 49 feature on average (15 times faster than full computation), while achieving similar generalization as the fully trained classifier (red, middle subfigure). On the right subfigure, when the boundary is applied to prediction, Attentive Pegasus achieves a lower error rate than the full computation, and less than half the error of the Budgeted Boundary (green).}
\label{fig:bbresults1}
\end{center}
\end{figure*}
}

\def \figurecurtailedpegasosB{
\begin{figure*}[t]
\begin{center}
\includegraphics[height=2.5in,width=5in]{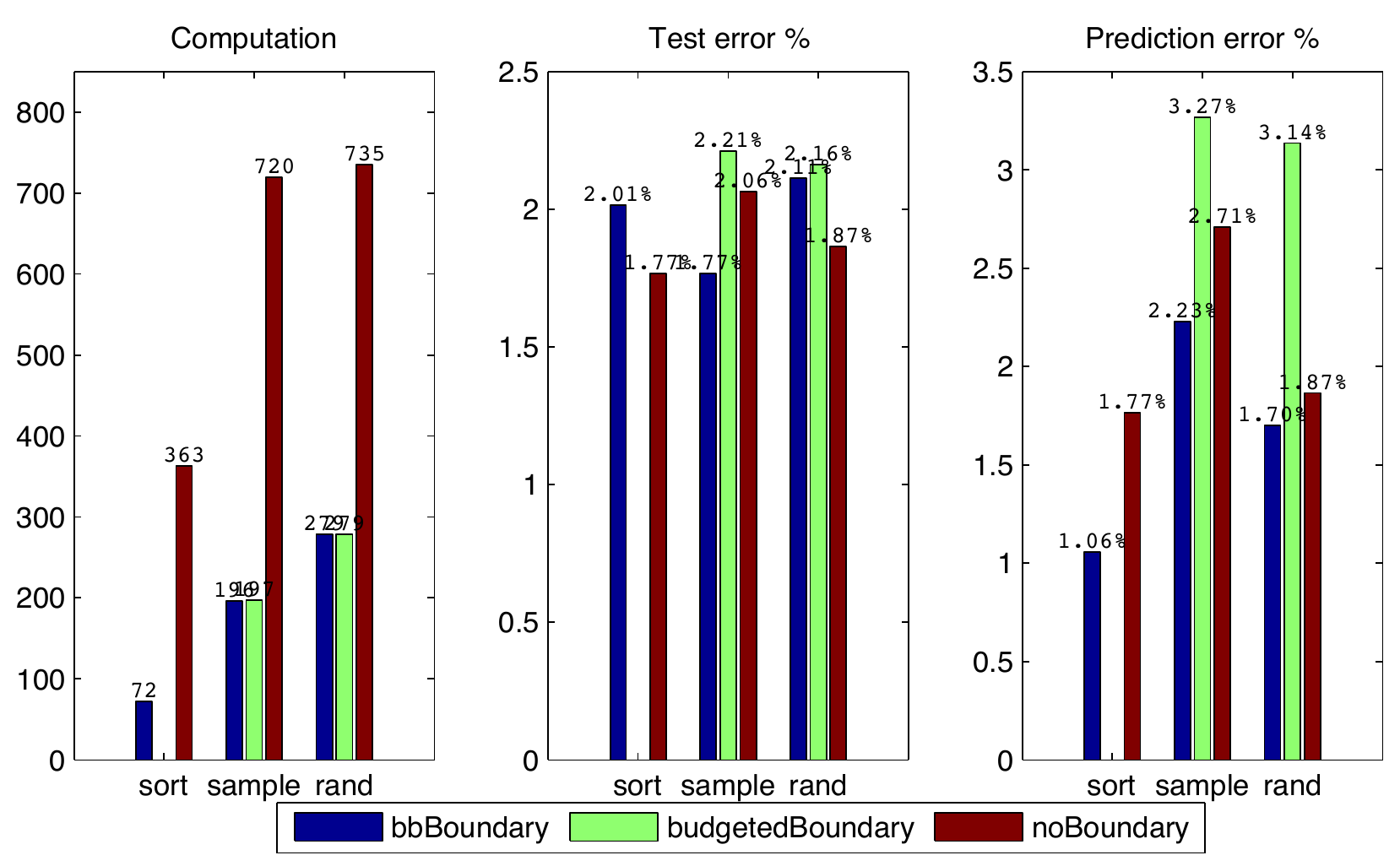}
\caption{Results for Attentive Pegasos. MNIST 3 vs 10, $\delta=10\%$. Our Brownian bridge decision boundary (blue) processes only 72 feature on average, while achieving similar generalization as the fully trained classifier. On the right, when the boundary is applied to classification, Attentive Pegasus gets a lower error rate than the full computation, and over a 2\% advantage over the Budgeted Boundary.}
\label{fig:bbresults2}
\end{center}
\end{figure*}
}

\icmltitlerunning{Attentive Learning and Prediction}

\begin{document} 

\twocolumn[
\icmltitle{Rapid Learning with Stochastic Focus of Attention}

\icmlauthor{Raphael Pelossof}{pelossof@cbio.mskcc.org}
\icmladdress{Comp. Bio. Sloan Kettering}
\icmlauthor{Zhiliang Ying}{zying@stat.columbia.edu}
\icmladdress{Statistics Department Columbia University}

\icmlkeywords{boring formatting information, machine learning, ICML}

\vskip 0.3in
]

\begin{abstract} 
We present a method to stop the evaluation of a decision making process when the result of the full evaluation is obvious. This trait is highly desirable for online margin-based machine learning algorithms where a classifier traditionally evaluates all the features for every example. We observe that some examples are easier to classify than others, a phenomenon which is characterized by the event when most of the features agree on the class of an example.
By stopping the feature evaluation when encountering an easy to classify example, the learning algorithm can achieve substantial gains in computation. Our method provides a natural attention mechanism for learning algorithms. By modifying Pegasos, a margin-based online learning algorithm, to include our attentive method we lower the number of attributes computed from $n$ to an average of $O(\sqrt{n})$ features without loss in prediction accuracy. We demonstrate the effectiveness of Attentive Pegasos on MNIST data.
\end{abstract} 

\section{Introduction}
The running time of margin based online algorithms is a function of the number of features, or the dimensionality of the input space. Since models today may have thousands of features, running time seems daunting, and depending on the task, one may wish to speed-up these online algorithms, by pruning uninformative examples. We propose to early stop the computation of feature evaluations for uninformative examples by connecting Sequential Analysis \cite{wald45tests,lan82sequential} to margin based learning algorithms.

Many decision making algorithms make a decision by comparing a sum of observations to a threshold. If the sum is smaller than a pre-defined threshold a certain action is taken, otherwise a different action is taken or no action is taken. This type of reasoning is prevalent in the margin-based Machine Learning community, where typically an additive model is compared to a threshold, and a subsequent action is taken depending on the result of the comparison. Margin-based learning algorithms average multiple weak hypotheses to form one strong combined hypothesis - the majority vote. When training, the combined hypothesis is usually compared to a threshold to make a decision about when to update the algorithm. When testing, the combined hypothesis is compared to a threshold to make a predictive decision about the class of the evaluated example.

With the rapid growth of the size of data sets, both in terms of the number of samples, and in terms of the dimensionality, margin based learning algorithms can average thousands of hypotheses to create a single highly accurate combined hypothesis. Evaluating all the hypotheses for each example becomes a daunting task since the size of the data set can be very large, in terms of number of examples and the dimensionality of each example. In terms of number of examples, we can speed up processing by filtering out un-informative examples for the learning process from the data set. The measure of the importance of an example is typically a function of the majority vote. In terms of dimensionality, we would like to compute the least number of dimensions before we decide whether the majority vote will end below or above the decision threshold. Filtering out un-informative examples, and trying to compute as few hypotheses as possible are closely related problems \cite{blum1997selection}. The decision whether an example is informative or not depends on the magnitude of the majority rule, that is the amount of disagreement between the hypotheses. Therefore, to find which example to filter, the algorithm needs to evaluate its majority vote, and we would like it to evaluate the least number of weak hypotheses before coming to a decision about the example's importance.

Majority vote based decision making can be generalized to comparing a weighted sum of random variables to a given threshold. If the majority vote falls below the pre-specified threshold a decision is made, typically the model is updated, otherwise the example is ignored. Since, the majority vote is a summation of weighted random variables, or averaged weak hypotheses, it can be computed sequentially. Sequential Analysis allows us to develop the statistical needed to speed up this evaluation process when its result is evident. 

We use the terms margin and full margin to describe the summation of all the feature evaluations, and partial margin as the summation of a part of the feature evaluations.
The calculation of the margin is broken up for each example in the stream. This break-up allows the algorithm to make a decision after the evaluation of each feature whether the next feature should also be evaluated or the feature evaluation should be stopped, and the example should be rejected for lack of importance in training.
By making a decision after each evaluation we are able to early stop the evaluation of features on examples with a large partial margin after having evaluated only a few features. Examples with a large partial margin are unlikely to have a full margin below the required threshold. Therefore, by rejecting these examples early, large savings in computation are achieved.

This paper proposes several simple novel methods based on Sequential Analysis \citep{wald45tests,lan82sequential} and stopping methods for Brownian Motion to drastically improve the computational efficiency of margin based learning algorithms. Our methods accurately stop the evaluation of the margin when the result is the entire summation is evident. 

Instead of looking at the traditional classification error we look at decision error. Decision errors are errors that occur when the algorithm rejects an example that should be accepted for training. Given a desired decision error rate we would like the test to decide when to stop the computation. This test is adaptive, and changes according to the partial computation of the margin. We demonstrate that this simple test can speed-up Pegasos by an order of magnitude while maintaining generalization accuracy. Our novel algorithm can be easily parallelized. 

\section{Related Work}
Margin based learning has spurred countless algorithms in many different disciplines and domains. Typically a margin based learning algorithm evaluates the sign of the margin of each example and performs a decision. Our work provides early stopping rules for the margin evaluation when the result of the full evaluation is obvious. This approach lowers the average number of features evaluated for each example according to its importance.
Our stopping thresholds apply to the majority of margin based learning algorithms. 

The most directly applicable machine learning algorithms are margin based online learning algorithms. 
Many margin based Online Algorithms base their model update on the margin of each example in the stream. Online algorithms such as Kivinen and Warmuth's Exponentiated Gradient \cite{kivinen97exponentiated} and Oza and Russell's Online Boosting \cite{oza01online} update their respective models by using a margin based potential function. Passive online algorithms, such as Rosenblatt's perceptron \cite{rosenblatt58perceptron} and Crammer etal's online passive-aggressive algorithms \cite{crammer06online}, define a margin based filtering criterion for update, which only updates the algorithm's model if the value of the margin falls below a defined threshold. All these algorithms fully evaluate the margin for each example, which means that they evaluate all their features for every example. Recently \citet{shalev2008svm,shalev2010pegasos} proposed Pegasos, an online stochastic gradient descent based SVM solver. The solver is a stochastic gradient descent solver which produces a maximum margin classifier at the end of the training process.

The above mentioned algorithms passively evaluate all the features for each given example in the stream. However, if there is a cost (such as time) assigned to a feature evaluation we would like to design an efficient learner which actively choose which features it would like to evaluate. Similar work on the idea of learning with a feature budget was first introduced to the machine learning community in \citet{bendavid1998learning}. The authors introduced a formal framework for the analysis of learning algorithm with restrictions on the amount of information they can extract. Specifically allowing the learner to access a fixed amount of attributes, which is smaller than the entire set of attributes. They presented a framework that is a natural refinement of the PAC learning model, however traditional PAC characteristics do not hold in this framework. Very recently, both \citet{cesabianchi2010efficient} and \citet{reyzin10boosting} studied how to efficiently learn a linear predictor under a feature budget (see figure \ref{fig:stochastic-budget}.) Also, \cite{clarkson2010sublinear} extended the Perceptron algorithm to efficiently learn a classifier in sub-linear time.
\figstochasticbudget

Similar active learning algorithms were developed in the context of when to pay for a label (as opposed to an attribute). Such active learning algorithms are presented with a set of unlabeled examples and decide which examples labels to query at a cost. The algorithm's task is to pay for labels as little as possible while achieving specified accuracy and reliability rates \cite{dasgupta05analysis, cesabianchi06sampling}.
Typically, for selective sampling active learning algorithms the algorithm would ignore examples that are easy to classify, and pay for labels for harder to classify examples that are close to the decision boundary.

Our work stems from connecting the underlying ideas between these two active learning domains, attribute querying and label querying. The main idea is that typically an algorithm should not query many attributes for examples that are easy to classify. The labels for such examples, in the label query active learning setting, are typically not queried. For such examples most of the attributes would agree to the classification of the example, and therefore the algorithm need not evaluate too many before deciding the importance of such examples. 

\section{The Sequential Thresholded Sum Test}
The novel \textit{Constant Sequential Thresholded Sum Test} is a test which is designed to control the rate of decision errors a margin based learning algorithm makes. Although the test is known in statistics, it have never been applied to learning algorithms before. 

\subsection{Mathematical Roadmap}
Our task is to find a filtering framework that would speed-up margin-based learning algorithms by quickly rejecting examples of little importance. Quick rejection is done by creating a test that stops the margin evaluation process given the partial computation of the margin. We measure the importance of an example by the size of its margin. We define $\theta$ as the importance threshold, where examples that are important to learning have a margin smaller than $\theta$. Statistically, this problem can be generalized to finding a test for early stopping the computation of a partial sum of independent random variables when the result of the full summation is guaranteed with a high probability.

We look at decision errors of a sum of weighted independent random variables. Then given a required decision error rate we will derive the \textit{Constant Sequential Thresholded Sum Test} (Constant STST) which will provide adaptive early stopping thresholds that maintain the required confidence.

Let the sum of weighted independent random variables $(w_i,X_i),
i=1,...,n$ be defined by $S_n = \sum_{i=1}^{n} w_i X_i$, where $w_i$
is the weight assigned to the random variable $X_i$. We require that
$w_i\in R, X_i\in[-1,1]$.
We define $S_n$ as the full sum, $S_i$ as the partial sum, and
$S_{in} = S_n-S_i = \sum_{j=i+1}^n w_i X_i$ as the remaining sum.
Once we computed the partial sum up to the $i$th random variable we
know its value $S_i$. Let the stopping threshold at coordinate $i$
be defined by $\tau_i$.  We use the notation $ES_{in}$ to denote the
expected value the remaining sum.

There are four basic events that we are interested in when designing sequential tests which involve controlling decision error rates. Each of these events is important for different applications under different assumptions. The sequential method looks at events that involve the entire random walks, whereas the curtailed method looks at evens that accumulate information as the random walk progresses. Let us establish the basic relationship between the sequential method, on the left hand side of the following equations, and the curtailed method on the right hand side:
\begin{equation}
P(\textit{stop} | S_n < \theta)P(S_n < \theta) = P(S_n < \theta | \textit{stop}) P(\textit{stop}),
\label{eqn:error_probabilities}
\end{equation}
where \textit{stop} is the event which occurs when the partial sum crosses a stopping boundary $\textit{stop}\doteq \{S_i > \tau_i\}$.

Previously \cite{x} a Curved STST was proposed by looking at the following ``curtailed" conditional probability
\begin{equation}
P(S_n < \theta | \textit{stop}) = \frac{P(S_n < \theta,\textit{stop}) }{P(\textit{stop})}.
\end{equation}
The simplicity of deriving the curtailed method stems from the fact that the joint probability and the stopping time probability are not needed to be explicitly calculated to upper bound this conditional. The resulting first stopping boundary, the curved boundary, gives us a constant conditional error probability throughout the curve, which meant that it is a rather conservative boundary. 

We develop a more aggressive boundary which allows higher decision error rates at the beginning of the random walk and lower decision error rates at the end. Such a boundary would essentially stop more random walks early on, and less later later on. This approach has the natural interpretation that we want to shorten the feature evaluation for obvious un-important samples, but we want to prolong the evaluations for samples we are not sure about. A Constant boundary achieves this exact ``error spending'' characteristic.

\subsection{The Constant Sequential Thresholded Sum Test}
We condition the probability of making a decision error in the following way
\begin{eqnarray}
P(\text{stop before } n | S_n < \theta) = \frac{P(\text{stop before } n , S_n < \theta)}{P(S_n < \theta)}.
\label{eqn:rare-decision-probability}
\end{eqnarray}
We stated in equation \ref{eqn:rare-decision-probability} a conditional probability function which is conditioned on the examples of interest. Therefore in this case we are interested in limiting the decision error rate for examples that are important. To upper bound this conditional we will make an approximation that will allow us to apply boundary-crossing inequalities for a Brownian bridge.
To apply the Brownian bridge to our conditional probability we approximate it by 
\begin{eqnarray}
\lefteqn{P(\text{stopped before } n  | S_n < \theta)} \\
&=& P(\max_i S_i > \tau | S_n < \theta) \\
&\approx&  P(\max_i S_i > \tau | S_n = \theta).
\end{eqnarray}
If we assume that the event $\{S_n < \theta\}$ is rare (equivalently that $EX_i>0$), then we can approximate the inequality with an equality, which gives a Brownian bridge. Now we need to calculate boundary crossing probabilities of the Brownian bridge and a constant threshold. 
\begin{lemma}
\label{thm:bb-upper}
The Brownian bridge Stopping Boundary. Let  $T_\tau = \inf\{i:S_i = \tau\}$ be the first hitting time of the random walk and constant $\tau$. Then the probability of the following decision error is 
\[P(T_\tau < n | S_n = \theta) = e^{-\frac{2\tau(\tau-\theta)}{var(S_n)}}. \]
\end{lemma}
\begin{proof}
See Appendix.
\end{proof}

\begin{theorem}
The Simplified Contant STST boundary ($\theta = 0$), $\tau = \sqrt{var(S_n)}\sqrt{log\frac{1}{\sqrt{\delta}}}$ makes approximately $\delta$ decision mistakes $\{T_\tau < n | S_n < 0\}$.
\end{theorem}
\begin{proof}
By approximating $P(S_i > \tau | S_n < \theta) \approx P(S_i > \tau | S_n = \theta)$ and setting this probability to $\delta$ we get the Contant STST boundary
\begin{equation}
P(S_i > \tau | S_n < \theta) \approx \exp\left\{ -\frac{2\tau(\tau-\theta)}{var(S_n)} \right\} = \delta.
\end{equation}
Solving,
\begin{eqnarray}
\tau^2 - \tau\theta &=& var(S_n)log\frac{1}{\sqrt{\delta}} \\
(\tau-\theta)^2 - \frac{1}{4}\theta^2 &=& var(S_n)log\frac{1}{\sqrt{\delta}} \\
\tau &=& \theta + \sqrt{\frac{1}{4}\theta^2 + var(S_n)log\frac{1}{\sqrt{\delta}}}.
\end{eqnarray}
If we simplify this boundary by setting $\theta$ to zero, we get the theorem's boundary
$$\tau = \sqrt{var(S_n)}\sqrt{log\frac{1}{\sqrt{\delta}}}.$$
\end{proof}
There are two appealing things about this boundary, the first that it's not dependent on $ES_i$, and that it is always positive.  Secondly, when using this boundary for prediction, we can directly see the implication on the error rate of the classifier, since the decision error essentially becomes a classification error, a fact that is also clearly evident throughout the experiments.
\figbbdecisionerror

\subsection{Average Stopping Time for the Curved and Constant STST}
We are able to show that the expected number of features evaluated for the Curved and the Constant STST boundaries is in the order of $O(\sqrt{n})$. This can be obtained by limiting the range of values $X_i$ can take.
\begin{theorem}
\label{thm:stopping-time}
Let $|X_i| \le k$, and let $EX_i > 0$. Let the stopping time of the Brownian bridge be defined by $t = \inf \{i: S_i \ge \sqrt{var(S_n)\log \delta^{-0.5}} \}$. Then the expected stopping time is in the order of $O(\sqrt n)$.
\end{theorem}
\begin{proof}
\begin{eqnarray}
ES_T &=& ES_{T-1} + EX_T \\
&\le& ES_{T-1} + k \\
&\le& \sqrt{var(S_n)\log \delta^{-0.5}} + k
\end{eqnarray}
The second inequality holds since the random walk only crossed the boundary for the first time at time $T$ and therefore was under the boundary at time $T-1$. By applying Wald's Equation $ES_T = ET EX$ we get
\begin{eqnarray}
ET &=& \frac{\sqrt{var(S_n)\log \delta^{-0.5}} + k}{EX} \\
&\le& \frac{c \sqrt{n} \sqrt{\log \frac{1}{\sqrt\delta}} + k}{EX} \\
&=& O(\sqrt n),
\end{eqnarray}
where $c,k$, and $EX$ are constants.
\end{proof}
See figure \ref{fig:bb_error_rate} for simulation results.

\section{Attentive Pegasos}
Pegasos by \cite{shalev2010pegasos} is a simple and effective
iterative algorithm for solving the optimization
problem cast by Support Vector Machines. To solve the SVM functional it alternates between stochastic
gradient descent steps (weight update) and projection steps (weight scaling). Similarly to the 
Perceptron these steps are taken when the algorithm makes
a classification error while training. For a
linear kernel, the total run-time of Pegasos is 
$\tilde O(d/(\lambda \epsilon))$, where $d$ is the number of 
non-zero features in each example, $\lambda$ is a regularization
parameter, and $\epsilon$ is the error rate incurred by the algorithm over the optimal error rate achievable. By assuming that the features are independent, and applying the Brownian bridge STST
we speed up Pegasos to $\tilde O(\sqrt{d}/(\lambda \epsilon))$ without losing
significant accuracy. The algorithm \textit{Attentive Pegasos} is demonstrated in Algorithm \ref{alg:pegasos}.

\subsection{Experiments}
We conducted several experiments to test the speed, generalization capabilities, and predictive accuracy of the STST. We ran Pegasos, Attentive Pegasos and Budgeted Pegasos on 1-vs-1 MNIST digit classification tasks. We also tested different sampling and ordering methods for coordinate selection.
\begin{algorithm}[h!]
   \caption{Attentive Pegasos}
   \label{alg:pegasos}
\begin{algorithmic}
\STATE {\bf Input}: $Dataset \{\mathbf{X}^l,y^l\}_{l=1}^m,\lambda,\delta$
\STATE {\bf Initialize}: Choose $\mathbf{w}_1$ s.t. $||\mathbf{w}_1||\le 1/\sqrt{\lambda}, j=0$
\FOR {$l=1,2,...,m$}
	\IF {$\exists i=1,..,n \text{ s.t. } y^l \sum_{j=1}^i w_i x_i^l \ge 1+\sqrt{\sum_{j=1}^n w_j \cdot var_{y^l}(x_j)}\sqrt{\log \delta^{-0.5}}$}
		\STATE Update $var_{y^l}(x_j), j=1,..,i$
		\STATE $\mathbf{w}^{l} = \mathbf{w}^{l-1}$
		\STATE Jump to next example
	\ELSE
		\STATE Set $\mu_l = \frac{1}{l\lambda}$
		\STATE Set $\mathbf{w}_{l+\frac{1}{2}} = (1-\mu_l\lambda)\mathbf{w}_l + \mu_l y\mathbf{x}$
		\STATE Set $\mathbf{w}_{l+1} = min\left\{ 1, \frac{1/\sqrt{\lambda}}{||\mathbf{w}_{l+\frac{1}{2}}||}\right\}$
	\ENDIF
\ENDFOR
\STATE 
\STATE return $\mathbf{w}^{m+1}$
\end{algorithmic}
\end{algorithm}

Attentive Pegasos stops the computation of examples that were unlikely to have a negative margin early. We computed the average number of features the algorithm computed for examples that were filtered, and compare it to a budgeted version of Pegasos, where only a fixed number of features are evaluated for any example. We ran 1-vs-1 digit classification problems under different feature selection policies. With the first policy, we sorted the coordinates, such that coordinates with a large absolute weight before other features with lesser absolute weight. Then with the second, we ran experiments where the coordinates were selected by sampling from the weight distribution with replacement. Finally, with the third, we ran experiments where the coordinates were randomly permuted.

For each one of these scenarios, three algorithms were run 10 times on different permutations of the datasets and their results were averaged (Figure \ref{fig:bbresults2}.) We first ran Attentive Pegasus under each of the coordinate selection methods. Then we set the budget for Budgeted Pegasus as the average number of features that we got through Attentive Pegasos. Finally, we ran the full Pegasus with a trivial boundary, which essentially computes everything. Both figures show that using the Brownian Bridge boundary can save in the order of 10x computation, and maintain similar generalization results to the full computation. Also, sorting under the Budgeted Pegasos is impossible since we need to learn the weights in order to sort them. Therefore we did not run Budgeted Pegasos with sorted weights. We can see in the middle subfigure, that the Attentive, Budgeted, and Full algorithms maintain almost identical generalization results. However, when we early stop prediction with the resulting model, Attentive Pegasos gives the best predictive results, even better than what we get with the full computation but only computes a tenth of the feature values!

\subsection{Conclusions}
We sped up online learning algorithms up to an order of magnitude without any significant loss in predictive accuracy. Surprisingly, in prediction Attentive Pegasos outperformed the original Pegasus, even though it. In addition, we proved that the expected speedup under independence assumptions of the weak hypotheses is $O(\sqrt{n})$ where $n$ is the set of all features used by the learner for discrimination.

The thresholding process creates a natural attention mechanism for online learning algorithms. Examples that are easy to classify (such as background) are filtered quickly without evaluating many of their features. For examples that are hard to classify, the majority of their features are evaluated. By spending little computation on easy examples and a lot of computation on hard ``interesting'' examples the Attentive algorithms exhibit a stochastic focus-of-attention mechanism.

It would be interesting to find an explanation for our results where the attentive algorithm outperforms the full computation algorithm in the prediction task.

\figurecurtailedpegasosA
\figurecurtailedpegasosB

\section{Appendix}
\label{sec:appendix}

\begin{lemma}
Brownian bridge boundary crossing probability
\begin{equation}
P(T_\tau < n | S_n = \theta)=\exp\left\{ -\frac{2\tau(\tau-\theta)}{var(S_n)} \right\}.
\end{equation}
\end{lemma}
\begin{proof}
We can look at an infinitesimally small area $d\theta$ around $\theta$
\begin{equation}
P(T_\tau \le n | S_n = \theta) = \frac{P(T_\tau < n , S_n = \theta)}{P(S_n = \theta)}.
\label{eqn:bb-lim}
\end{equation}
The numerator can be developed to
\begin{eqnarray}
\lefteqn{P(T_\tau < n , S_n = \theta)}\\
&=& P(T_\tau < n) P(S_n \in d\theta | T_\tau<n) \\
&=& P(T_\tau < n) P(S_n \in 2\tau-d\theta | T_\tau<n) \\
&=& P(S_n \in 2\tau-d\theta, T_\tau < n) \\
&=& P(S_n \in 2\tau-d\theta) \\
&=& \frac{1}{\sqrt{var(S_n)}}\phi\left( \frac{2\tau-\theta}{\sqrt{var(S_n)}}\right)d\theta
\end{eqnarray}
Similarly, the denominator
$$\text{denominator} = \frac{1}{\sqrt{var(S_n)}}\phi\left( \frac{\theta}{\sqrt{var(S_n)}}\right)d\theta$$
Plugging back into \ref{eqn:bb-lim}
\begin{eqnarray}
\lefteqn{P(T_\tau < n | S_n = \theta)}\\
&=& \frac{\phi\left( \frac{2\tau-\theta}{\sqrt{var(S_n)}}\right)}{\phi\left( \frac{\theta}{\sqrt{var(S_n)}}\right)} \\
&=& \exp\left\{ -\frac{1}{2} \frac{(2\tau-\theta)^2}{{var(S_n)}} + \frac{1}{2}\frac{\theta^2}{{var(S_n)}} \right\} \\
&=& \exp\left\{ -\frac{2\tau(\tau-\theta)}{{var(S_n)}} \right\}.
\end{eqnarray}
\end{proof}
\vspace{-3em}
\begin{definition}
\label{def:stopping-rv}
A random variable $T$ which is a function of $X_1,X_2,...$ is a stopping time if $T$ has nonegative integer values and for all $n=1,2,...$ there is an event $A_n$ such that $T\le n$ if and only if $(X_1,...,X_n)\in A_n$, while for $n=0,\{T=0\}$ is either empty (the usual case) or the whole space.
For a nonnegative integer-valued random variable $X$, we have
\begin{eqnarray}
\lefteqn{\sum_{j=1}^\infty P(X\ge j)}\\
&=&\sum_{j=1}^\infty \sum_{k\ge j}P(X=k) \nonumber\\
&=&\sum_{k=1}^\infty P(X=k)\sum_{j=1}^k 1 \nonumber \\
&=&\sum_{k=1}^\infty kP(X=k) = EX \le +\infty
\end{eqnarray}
\end{definition}
\begin{lemma}(\citep{wald45tests}, proof from \cite{dudley2010random})
{\bf Wald's Identity}.
Let $S_T$ be a sum of independent identically distributed random variables $X_1+...+X_T$, where $EX_i<\infty$. Let $S_n = X_1+...+X_n, i>T$ and $S_0=0$. Let $T=\{\inf_i S_i = a\}, T>0$ where $a$ is a constant be a random variable with $ET<\infty$. Then $ES_T = ETEX$.
\end{lemma}
\begin{proof}
If $T=0$ the identity is trivial. Otherwise
\begin{eqnarray}
ES_T &=& \sum_{n=1}^\infty P(T=n)E(X_1+...+X_n|T=n)\\
&=&\sum_{n=1}^\infty P(T=n)\sum_{j=1}^n E(X_j|T=n)\\
&=&\sum_{j=1}^\infty\sum_{n=j}^\infty P(T=n) E(X_j|T=n)\\
&=&\sum_{j=1}^\infty\sum_{n=j}^\infty E(X_j{\bf 1}_{T=n})\\
&=&\sum_{j=1}^\infty E(X_j{\bf 1}_{T\ge j})\\
&=&\sum_{j=1}^\infty E(X_j(1-{\bf 1}_{T\le{j-1}}).
\end{eqnarray}
The event $\{T\le j-1 \}$ is independent of $Y_j$, therefore
\begin{eqnarray}
ES_T &=& \sum_{j=1}^\infty E(X_j)P(T\ge j) \\
&=& EX\sum_{j=1}^\infty P(T\ge j) = EXET
\end{eqnarray}
by definition \ref{def:stopping-rv}.
\end{proof}

\bibliography{bibdesk}
\bibliographystyle{icml2011}

\end{document}